\documentclass[journal]{IEEEtran}
\usepackage{amsmath}
\usepackage[pdftex]{graphicx}
\usepackage{url}
\usepackage{multirow}
\usepackage{amssymb}
\usepackage{color}
\usepackage{bm}
\usepackage{caption}

\usepackage{algorithm}
\usepackage{algorithmic}

\newtheorem{proposition}{Proposition}
\newtheorem{proof}{Proof}

\usepackage{cite}
\newcommand{\ie}{\textit{i.e., }}
\newcommand{\eg}{\textit{e.g., }}
\newcommand{\tabincell}[2]{
\begin{tabular}{@{}#1@{}}#2\end{tabular}
}
\newcommand{\temp}[1]{{\color{black}#1}}  
\newcommand{\bihan}[1]{{\color{black}#1}}  

\ifCLASSOPTIONcompsoc
\usepackage[caption=false, font=normalsize, labelfont=sf, textfont=sf]{subfig}
\else
\usepackage[caption=false, font=footnotesize]{subfig}
\fi

\begin{document}

\title{Reconciliation of Statistical and Spatial Sparsity For Robust Image and Image-Set Classification}
\author{Hao~Cheng, Kim-Hui~Yap,~\IEEEmembership{Member,~IEEE} and~Bihan~Wen,~\IEEEmembership{Member,~IEEE}
\thanks{H. Cheng, K-H. Yap and B. Wen are with the School of Electrical and Electronic Engineering, Nanyang Technological University, 639798 Singapore, e-mails:hao006@e.ntu.edu.sg, (ekhyap, bihan.wen)@ntu.edu.sg}.
}

\markboth{Journal of \LaTeX\ Class Files,~Vol.~14, No.~8, August~2015}%
{Shell \MakeLowercase{\textit{et al.}}: Bare Demo of IEEEtran.cls for IEEE Transactions on Magnetics Journals}

\IEEEtitleabstractindextext{%
\begin{abstract}
Recent image classification algorithms, by learning deep features from large-scale datasets, have achieved significantly better results comparing to the classic feature-based approaches. 
However, there are still various challenges of image classifications in practice, such as classifying noisy image or image-set queries, and training deep image classification models over the limited-scale dataset.
Instead of applying generic deep features, the model-based approaches can be more effective and data-efficient for robust image and image-set classification tasks, as various image priors are exploited for modeling the inter- and intra-set data variations while preventing over-fitting.
In this work, we propose a novel Joint Statistical and Spatial Sparse representation, dubbed \textit{J3S}, to model the image or image-set data for classification, by reconciling both their local patch structures and global Gaussian distribution mapped into Riemannian manifold.
To the best of our knowledge, no work to date utilized both global statistics and local patch structures jointly via joint sparse representation.
We propose to solve the joint sparse coding problem based on the J3S model, by coupling the local and global image representations using joint sparsity.
The learned J3S models are used for robust image and image-set classification.
Experiments show that the proposed J3S-based image classification scheme outperforms the popular or state-of-the-art competing methods over FMD, UIUC, ETH-80 and YTC databases.
\end{abstract}

}
\maketitle

\IEEEdisplaynontitleabstractindextext

\IEEEpeerreviewmaketitle

\section{Introduction} \label{sec1}
\IEEEPARstart{I}{mage} classification is a fundamental problem in image processing and computer vision.
Comparing to classic algorithms based on pre-defined features, recent image classification schemes applied machine learning techniques to optimize feature representation directly from the data themselves. 
More recently, deep learning approaches for image classification have achieved the state-of-the-art results on many benchmarking datasets, such as the popular ImageNet~\cite{deng2009imagenet}.
Despite of the promising performance achieved under simple and ideal problem setups, there are still various challenges when (i) classification based on queries that contain a set of object variations (\ie image-set classification), or (ii) the image data is limited or with relatively low quality (\ie weakly supervised classification).

To be specific, while conventional classification tasks process a single image in each query, image-set classification~\cite{ytc,eth80,gao2015patch,wang2020graph} has recently gained more attention, in which each query set contains multiple images with strong correlation (e.g., query object with multiple views, poses or illuminations).
Such type of algorithms are widely applied in applications such as video-based face classification~\cite{ytc}, multi-spectral image classification, etc. 
Comparing to the single-image classification algorithms, effective image-set methods need to additionally exploit the hidden structure among image sets, e.g., the inter- and intra-set data variations.
Furthermore, popular deep features tend to be generic and incorporate very little prior knowledge by learning from large-scale, high-quality, and fully annotated training datasets~\cite{deng2009imagenet}. 
Such approaches are ideal with fully supervised learning, but less data-efficient and robust when training sets are small-scale or corrupted (\eg noisy).

Recent works on sparse signal modeling have demonstrated their effectiveness in image representation for various tasks~\cite{zhu2014fast,gao2014concurrent,feng2016kernel,li2016image,cherian2016riemannian,wang2020learning,jing2020learning}.
Comparing to the deep features, sparse representation is model-based, thus much more robust to practical challenges such as noise or over-fitting~\cite{wen2017frist}.
While many existing works focused on exploiting image patch-based sparsity, global statistical properties are typically ignored or failed to be incorporated jointly in a principled approach.
Recent works show that high-order statistics of image features are critical in classification tasks~\cite{Kobayashi2014Dirichlet,Li2015From,li2017second,nguyen2020prominent}, leading to better results comparing to many first-order methods.

In this work, we propose a novel Joint Statistical and Spatial Sparse representation (J3S), \ie learning the coupled dictionaries for both local patch features, and the global data Gaussian distribution mapped into Riemannian manifold. 
Their dictionary-domain sparse coefficients are reconciled by solving a sparse coding problem with joint sparsity.
We propose an efficient yet effective alternating minimization algorithm to solve the J3S sparse coding problem.
To the best of our knowledge, no work to date utilized both global statistics and local patch structures jointly via sparse representation for image classification.
Furthermore, we apply the learned J3S model for robust image-set and single-image classification applications.
Extensive experimental results on material classification, object recognition and video-based face recognition tasks are presented, and we demonstrate that the proposed J3S-based image classification scheme outperforms the popular or state-of-the-art competing methods.

In short, the contributions of this paper include:
\begin{itemize}
 
 \item Learning global statistical and local patch dictionaries for visual classification task by coupling them with joint sparsity;

 \item Utilizing principal component analysis (PCA) to reduce the J3S model complexity while maintaining the effectiveness;

 \item Investigating the robustness of the proposed model under various conditions \ie noisy condition and few-shot setting.

 \item Achieving the state-of-the-art results on both noisy image and image-set classification tasks.
\end{itemize}

The remainder of this article is organized as follows.
Section~\ref{section2} summarizes the related work on image or image-set classification problems, including manifold learning, deep learning and sparse representation.
Section~\ref{section3} introduces two kinds of dictionary learning methods based on Gaussian-based statistical information and patch-based spatial information respectively, the proposed J3S model and the classification module.
Section~\ref{section4} describes the solution of the proposed J3S model based on alternation minimization and analyzes the time and space complexity as well as a simple strategy for model acceleration effectively.
Section~\ref{section5} demonstrates the performance of the proposed J3S model for image and image-set tasks over several standard databases under different conditions such as noise and few-shot.
Section~\ref{section6} concludes with proposals for future work.
The preliminary work has appeared in~\cite{cheng2020joint}.
\footnote{Significant changes have been made compared to our previous work in~\cite{cheng2020joint}. First, we \bihan{improve the J3S method by reducing the model complexity with a simple and efficient approach}.
Second, we add more description and \bihan{analysis on the dictionary learning and classification}.
Third, we include new experiments on ablation study to investigate the model convergence and parameter selection.
Furthermore, we conduct an extra experiment on an object recognition task based on the ETH-80 database to evaluate the generalizability of the proposed J3S model in different scenarios.
Finally, we conduct some additional popular settings, \eg noisy condition and few-shot setting are conducted to validate the performance and robustness of the proposed J3S model.}

\section{Related Work} \label{section2}

Image-set classification aims to identify the common class of a multi-image query. The inherent properties of each query set can be modeled effectively by popular methods such as manifold learning, deep learning, sparse coding, etc.

\textbf{Manifold Learning}: The classic methods based on Discriminant Canonical Correlations (DCC) \cite{kim2007discriminative} proposed to classify image sets by maximizing the canonical correlations of within-class sets and minimizing the canonical correlations of between-class sets.
Later on, more subspace methods \cite{cevikalp2010face} were proposed to simplify the geometric structure learning for image sets.
However, these approaches are limited as most image sets lie on a Riemannian manifold rather than Euclidean subspaces~\cite{wang2012covariance,huang2015log}, e.g., symmetric positive definite (SPD) manifold is widely used to represent image sets.
To ease the computation, the Log-Euclidean Riemannian Metric (LERM) framework \cite{arsigny2007geometric} proposed to map data from SPD manifold to its tangent Euclidean space.
Besides, Log-Euclidean Manifold Learning (LEML) \cite{huang2015log} projects the original SPD manifold to a lower-dimension discriminative SPD manifold while preserving its original geometry.
More recently, Riemannian Manifold Metric Learning (RMML) \cite{zhu2018towards} proposed a more generalized metric learning method which can be applied to multiple manifolds.
From a statistical perspective, when modeling image sets or the multi-channel features via Gaussian distribution, their covariance matrices for a collection of Gaussian can form a Riemannian manifold of SPD matrices~\cite{wang2012covariance,wang2016raid,wang2017g}.
Covariance Discriminative Learning (CDL) \cite{wang2012covariance} derived a Riemannian kernel function to map covariance matrix from manifold space to Hilbert space, where kernelized linear methods can be used for learning.

\textbf{Deep Learning}: Recently, more works on deep learning have shown its capability for image-set classification~\cite{hayat2014deep,Lu_2015_CVPR,sun2017learning}.
Deep Reconstruction Model (DRM) \cite{hayat2014deep} learns a template deep reconstruction model using neural networks and then uses the minimal reconstruction residual to classify a query set. 
Multi-manifold deep learning (MMDML) \cite{Lu_2015_CVPR} maps multiple sets of image into a shared feature subspace to leverage the nonlinear information.
More recently, Deep Match Kernels (DMK) \cite{sun2017learning} is proposed for image-set classification without considering specific assumptions on image distribution and geometrical structures and builds local match kernels to leverage its generic deep features.

\begin{table}[!t]
\centering
\fontsize{7.5}{14pt}\selectfont
\begin{tabular}{|c|c|c|c|c|}
\hline  
\textit{Methods} & \tabincell{c}{Feature \\ Learning} & \tabincell{c}{Model \\based} & \tabincell{c}{Manifold\\ Space} & \tabincell{c}{Small-Scale\\ Training}\\
\hline 
DCC \cite{kim2007discriminative} &   & \checkmark  &  & \checkmark \\
\hline 
AHISD/CHISD \cite{cevikalp2010face} & \checkmark & \checkmark &  & \checkmark  \\
\hline
LEML \cite{huang2015log}&  & \checkmark & \checkmark & \checkmark   \\
\hline 
RMML \cite{zhu2018towards}&  & \checkmark & \checkmark & \checkmark  \\
\hline
CDL \cite{wang2012covariance}&  & \checkmark & \checkmark & \checkmark  \\
\hline
RSR \cite{harandi2012sparse}& \checkmark & & \checkmark & \checkmark\\
\hline
KGDL \cite{harandi2013dictionary} & \checkmark &  & \checkmark & \checkmark\\
\hline
DRM \cite{hayat2014deep}& \checkmark & \checkmark &  &   \\
\hline
MMDML \cite{Lu_2015_CVPR}& \checkmark & \checkmark & \checkmark &   \\
\hline
DMK \cite{sun2017learning}& \checkmark & \checkmark &  &   \\
\hline
Proposed & \multirow{2}{*}{\checkmark} & \multirow{2}{*}{\checkmark} & \multirow{2}{*}{\checkmark} & \multirow{2}{*}{\checkmark}  \\
J3S & &  &  &  \\
\hline
\end{tabular}
\caption{Comparison of the key attributes between the proposed J3S method, and other image-set classification algorithms.}
\label{related}
\end{table}

\textbf{Sparse Representation}:
Sparse coding based classification represents a query sample on a dictionary composed of the training samples of all classes, and then classified by the reconstruction error of each class \cite{wright2009robust,kang2011feature,wen2015structured,wang2020hardness}.
Besides, the sparse coefficients can be used as the extracted features for classification, e.g., linear spatial pyramid matching \cite{yang2009linear}.
Most existing works focus on sparse coding and dictionary learning on the zero-order information, \ie the original feature space, while the first-order and second-order statistics contain global information and take the correlation of the data into account.
They can be more robust to variations in images and videos applications, e.g., variations of poses, illumination and occlusions.

Table~\ref{related} summarized the aforementioned related methods, as well as the proposed J3S method for image-set classification. 
Furthermore, some recent works also proposed sparse coding and dictionary learning models on Riemannian manifold of SPD matrices and Grassmann manifold:
Sparse coding on Riemannian manifold can be converted to a kernel sparse coding problem by deriving valid kernels for SPD manifold \cite{harandi2012sparse,cherian2016riemannian} or Grassmann manifold \cite{harandi2013dictionary}. 
However, none of the existing works combined statistical with spatial priors in the sparse representation.
Besides, the robustness of the image-set classification has been rarely investigated.

\section{Dictionary Construction and Joint Sparse Representation} \label{section3}

In this section, we present the J3S model for classification tasks, including the dictionary construction of statistical and spatial models and joint sparse coding.
Our proposed J3S model can deal well with different types of input data such as single image and image set.

To obtain the unified feature representations for classifying both a single image and an image data set, we apply the corresponding data preprocessing methods.
Specifically, for an image set ${\bf{M}}_i$ with feature of each image $\left\{x_1,x_2,...,x_n\right\}, x_j \in \mathbb{R}^{d}$, we combine them to construct the image set representation ${\bf{X}}_i$ directly; For a single image ${\bf{N}}_i$, we employ its deep feature representation $f({\bf{N}}_i) \in \mathbb{R}^{w \times h\times c}$ using a pre-trained CNN extractor as local features to construct ${\bf{X}}_i \in \mathbb{R}^{d \times c}$ where $d = w \times h$.
Thus, both an image or image set can be represented in a similar form as $\mathbf{X}_{i}=\left\{ \boldsymbol{x}_1, \boldsymbol{x}_2, ..., \boldsymbol{x}_{m_i} \right\} \in \mathbb{R}^{d \times m_i}$, where $d$ is the feature dimension of each image and $m_i$ is the number of image in each image set or the number of channel for a single image.

\subsection{Statistical Dictionary Construction} \label{sec31}

Based on the Gaussian statistical model, we need to compute the mean vector $\boldsymbol{\mu_i}$ and covariance matrix ${\mathbf{C}}_i$ in Reproducing Kernel Hilbert Space (RKHS) for the corresponding Gaussian descriptor ${G_i}\left( {\boldsymbol{\mu_i}},{\mathbf{C}}_i \right), i=1,2,...,N$.
We map $\mathbf{X}_i$ into an RKHS by the mapping function $\eta(\cdot)$ with Hellinger's kernel, the mean vector $\boldsymbol{\mu_i}$ and covariance matrix ${\mathbf{C}}_i$ can be computed as:
\begin{equation}
\boldsymbol{\mu_i} =\frac{1}{m_i} \sum_{k=1}^{m_i} \eta\left(\boldsymbol{x}_k\right), {\mathbf{C}_i}=\frac{1}{m_i} \Phi(\mathbf{X}_i) \mathbf{J} \Phi(\mathbf{X}_i)^{T}.
\label{cov}
\end{equation}
Here $\Phi(\mathbf{X}_i)=[\eta(\boldsymbol{x}_1),...,\eta(\boldsymbol{x}_{m_i})]$, and $\mathbf{J}=\mathbf{I}_{d}-\frac{1}{d}\mathbf{1}_{d}\mathbf{1}_{d}^{T}$ is the centering matrix.
However, when the dimension of the original features (\ie $d$) is very high, and the number of samples (\ie $m_i$) is small, such a Gaussian descriptor ${G_i}\left( {\boldsymbol{\mu_i}},{\mathbf{C}}_i \right)$ can not work well.
To solve this problem, following \cite{wang2016raid}, we estimate the robust covariance matrix $\widehat{\mathbf{\Sigma}}_i$ by solving a regularized maximum likelihood estimation problem as:
\begin{equation}
\min _{\widehat{\Sigma}_i} \log |\widehat{\mathbf{\Sigma}}_i|+\operatorname{tr}\left(\widehat{\mathbf{\Sigma}}_i^{-1} \widehat{\mathbf{C}}_i\right)+\alpha D_{\mathrm{vN}}(\mathbf{I}_{d}, \widehat{\mathbf{\Sigma}}_i),
\label{vnMLE}
\end{equation}
where $D_{\mathrm{vN}}$ is the von Neumann matrix divergence~\cite{kulis2009low} of two matrices and $\alpha \in (0,1)$ is a regularizing parameter.
The optimal solution of problem (\ref{vnMLE}) can be computed as:
\begin{equation}
\begin{aligned}
\widehat{\mathbf{\Sigma}}_i &=\widehat{\mathbf{U}}_i \operatorname{diag}\left(\lambda_{k}\right) \widehat{\mathbf{U}}_i^{T}, \\
\lambda_{k} &=\sqrt{\left(\frac{1-\alpha}{2 \alpha}\right)^{2}+\frac{\delta_{k}}{\alpha}}-\frac{1-\alpha}{2 \alpha}.
\end{aligned}
\end{equation}
Here $\delta_{k}$ is the diagonal matrix of the singular values in decreasing order, and $\widehat{\mathbf{U}}_i$ is the orthogonal matrix consisting of the eigenvectors corresponding to the singular values. $\widehat{\mathbf{U}}_i$ and $\delta_{k}$ are computed by the singular value decomposition (SVD) of the covariance matrix $\mathbf{C}_i$ as $svd(\mathbf{C}_i)=\widehat{\mathbf{U}}_i \operatorname{diag}\left(\delta_{k}\right) \widehat{\mathbf{U}}_i^{T}$.

By using the mean vector $\mu_i$ and robust covariance matrix $\widehat{\mathbf{\Sigma}}_i$, we can define the embedding symmetric positive definite matrix ${\bf{P}}_i$ as:
\begin{equation}
{\bf{P}}_i=\left[\begin{array}{cc}
{\widehat{\mathbf{\Sigma}}_i+\beta^{2} \boldsymbol{\mu_i} \boldsymbol{\mu_i}^{T}} & {\beta \boldsymbol{\mu_i}} \\
{\beta \boldsymbol{\mu_i}^{T}} & {1}
\end{array}\right] \in \mathbb{R}^{(d+1) \times (d+1)},
\label{pp}
\end{equation}
where $\beta >0$ is a parameter to balance the orders of magnitude between them.

\begin{figure*}[t]
\centering
\includegraphics[width=0.9\textwidth]{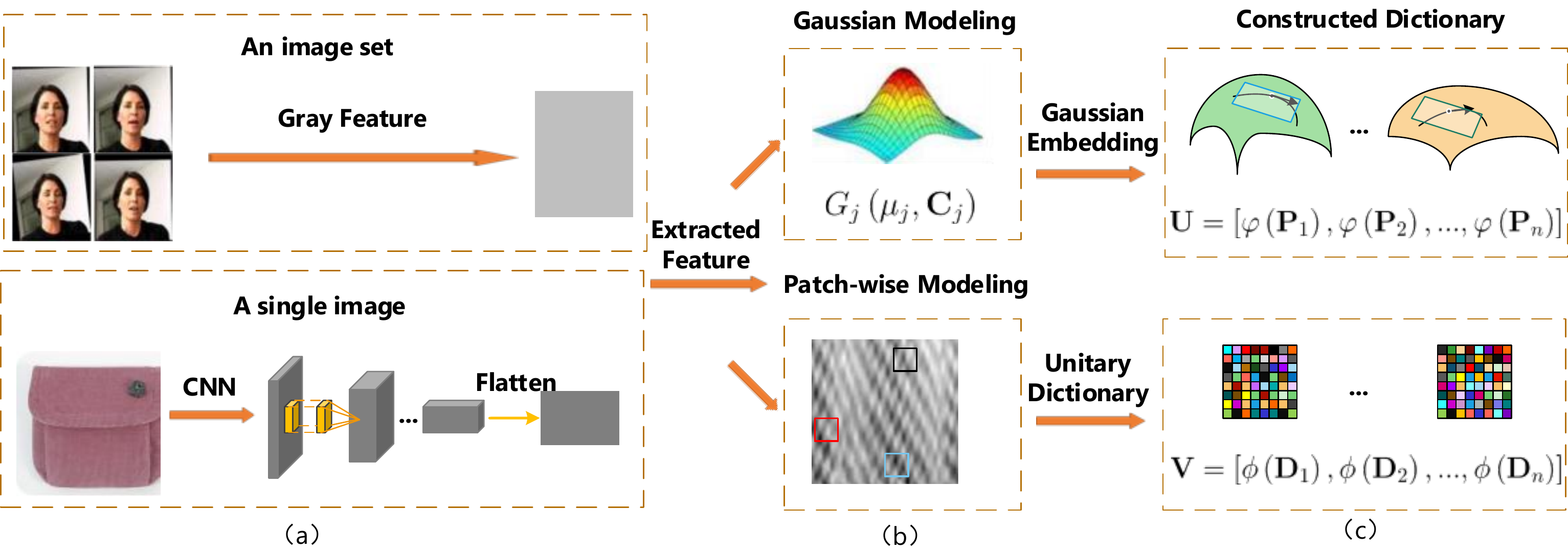}
\caption{The framework of Joint Statistical and Spatial Sparse representation of image and image set classification.}
\label{review}
\end{figure*}
\subsection{Spatial Dictionary Construction} \label{sec32}

Similarly, given a sample $\mathbf{X}_i$, considering spatial information, we can learn a patch-based unitary dictionary ${\mathbf{D}}_i$ from the feature map of a single image or gray feature of an image set.
For a single image, we choose the original image or the deep feature as the input to learn a patch-based unitary dictionary.
In contrast, for an image set, we combine each single image feature to construct a patch-based unitary dictionary to exploit within the class structure.
For image and image set classification, the objective is to learn a unitary dictionary $\mathbf{D}_i\in \mathbb{R}^{p \times p}$ based on 2D image patches constructed from sample $\mathbf{X}_i$ by solving the following problem with the synthesis model as:
\begin{equation}
\{\hat{\mathbf{D}}_i, \hat{\gamma_k}\}=\underset{\mathbf{D}_i, \{\gamma_k\}}{\operatorname{argmin}} \sum_{k=1}^{N}\left\|\mathbf{C}_k \boldsymbol{u}_i-\mathbf{D}_i \gamma_{k}\right\|_{2}^{2}, \text { s.t. } \mathbf{D}_i^{H} \mathbf{D}_i=\mathbf{I}_{p},
\label{unid}
\end{equation}
where $\mathbf{C}_k \boldsymbol{u}_i$ is used for extracting patches from $\mathbf{X}_i$, $\boldsymbol{u}_i$ is the vectorized form of $\mathbf{X}_i$, $N$ is the number of total patches and $\mathbf{I}_{p}$ is the identity matrix.


Sparse coding problems under the synthesis model are NP-hard in general, and even the approximate algorithms are typically expensive~\cite{pati1993orthogonal}.
However, since problem (\ref{unid}) learns the unitary dictionary, it is equivalent to the unitary transform learning problem~\cite{wen2020set}, \ie a signal $u$ is approximately sparsifiable using a learned unitary transform $\mathbf{W}_{i} \in \mathbb{R}^{m \times p}$, as $\mathbf{W}_{i}u=\gamma + \mathit{e}$, where $\gamma \in \mathbb{R}^p$ is sparse and $\mathit{e}$ is a small residual in the transform domain.
The corresponding transform learning problem is formulated as
\begin{equation}
\hat{\mathbf{W}}_i=\underset{\mathbf{W}_i}{\operatorname{argmin}} \sum_{k=1}^{N}\left\|\mathbf{W}_i \mathbf{C}_k \boldsymbol{u}_i-\gamma_{k}\right\|_{2}^{2}. \text { s.t. } \mathbf{W}_i^{H} \mathbf{W}_i=\mathbf{I}_{p}.
\label{uni}
\end{equation}
Based on~\cite{wen2020set}, the two sparsity models can be unified under the \textit{unitary dictionary assumption}, \ie $\mathbf{D}_i = \mathbf{W}_i^{T}$, and $\mathbf{D}_i^{T} \mathbf{D}_i = \mathbf{I}_{p}$. 

\begin{proposition}
Under the \textit{unitary dictionary assumption}, the problems (\ref{unid}) and (\ref{uni}) are equivalent.
\end{proposition}

\begin{proof}
Based on the \textit{unitary dictionary assumption}, we have $\mathbf{W}_i \mathbf{D}_i = \mathbf{I}_{p}$ and $\|\mathbf{W}_i \tau\|_{2}=\|\tau\|_{2}, \forall \tau .$
Thus, the function in problem (\ref{unid}) is identical to that in problem (\ref{uni}), \ie $\left\|\mathbf{C}_{k} \boldsymbol{u}_{i}-\mathbf{D}_{i} \gamma_{k}\right\|_{2}^{2}=\left\|\mathbf{W}_{i} \mathbf{C}_{k} \boldsymbol{u}_{i}-\gamma_{k}\right\|_{2}^{2}$.
Therefore, the problems (\ref{unid}) and (\ref{uni}) become equivalent, and $\hat{\mathbf{D}}_i=\hat{\mathbf{W}}_i^{T}$.
\end{proof}

Thus, we can obtain the optimal dictionary $\mathbf{D}_i=\hat{\mathbf{W}}_i^{T}$ in (\ref{unid}) by solving its equivalent problem (\ref{uni}) which has an exact and closed-form solution~\cite{ravishankar2015sparsifying}, \ie 
$\hat{\mathbf{W}}_i=\mathbf{G} \mathbf{S}^H$ where $\mathbf{G}, \mathbf{S}$ are computed by the SVD of $\mathbf{K} \triangleq \sum_{k=1}^{N}\left(\mathbf{C}_{k} \boldsymbol{x}_i\right) \gamma_{k}^H$ as $svd(\mathbf{K})=\mathbf{S} \Sigma \mathbf{G}^H$.

Fig.~\ref{review} illustrates the framework of J3S, in which a sample ${\hat{\mathbf{X}}}_j$ (either an image set or a single image) can be modelled by a statistical model to obtain the embedding SPD matrix $\mathbf{P}_j$, and simultaneously modelled by a spatial patch-based model to generate a unitary transform dictionary $\mathbf{D}_j$, which are then used for joint sparse coding.

\begin{figure*}[!t]
\centering
\includegraphics[width=1.0\textwidth]{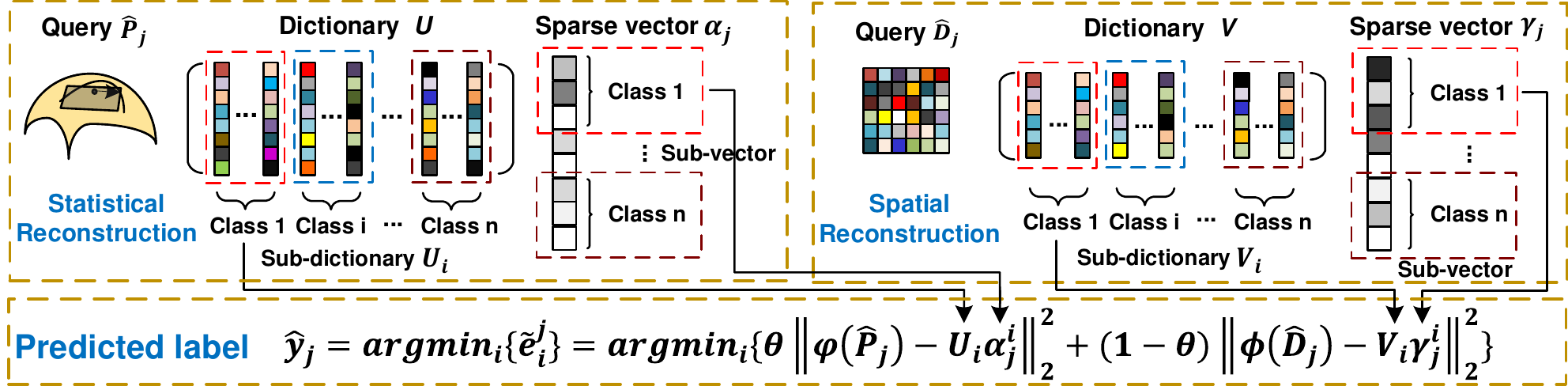}
\caption{The framework of classification module of Joint Statistical and Spatial Sparse representation model.}
\label{classification}
\end{figure*}

\subsection{J3S Sparse Coding} \label{sec33}

To reconcile the two types of dictionaries generated from statistical Gaussian modeling and spatial patch-wise unitary dictionary learning, we propose the joint statistical and spatial sparse representation (J3S) model, \ie for any query sample, we impose joint sparsity on their statistical and spatial dictionary-domain coefficients, to maintain the consistency and dependency of their individual sparse representation.
For simplicity, we use the following symbols to simplify the objective function,
\[
\begin{split}
\begin{aligned}
{\bf{U}_i} &= \left[ {\varphi \left( {{{\mathbf{P}}_{i1}}} \right),\varphi \left( {{{\bf{P}}_{i2}}} \right),...,\varphi \left( {{{\bf{P}}_{ik}}} \right)} \right],\\
{\bf{V}_i} &= \left[ {\phi \left( {{{\mathbf{D}}_{i1}}} \right),\phi \left( {{{\mathbf{D}}_{i2}}} \right),...,\phi \left( {{{\mathbf{D}}_{ik}}} \right)} \right]. 
\end{aligned}
\end{split}
\]
Here, $\varphi \left( {{{\bf{P}}_{{i_1...k}}}} \right)$ and $\phi \left( {{{\bf{D}}_{{i_1...k}}}} \right)$ are the mapping function of Gaussian and patch-based unitary dictionary corresponding to the $k$ training samples belonging to the $i$-th class, respectively.

For statistical Gaussian-based dictionary, each embedding matrix $\mathbf{P}_i$ is an SPD matrix which can be viewed as a point on the corresponding SPD manifold based on Eq (\ref{pp}).
Direct vectorization of the SPD matrix to generate a dictionary will destroy the intrinsic structure, which may cause information loss.
To avoid such loss while measuring the similarity between two matrices on the SPD manifold, we use a general framework called LERM~\cite{arsigny2007geometric} to map each of them to its tangent space through matrix logarithm operation $log(\cdot)$.
By using such embedding, we directly measure the similarity on the tangent space using the Euclidean distance.
In this way, we can get the vectorized form of each matrix as the statistical Gaussian-based feature for feature representation.
To simplify the calculation of SPD matrix $log(\mathbf{P}_i)$, here we only extract upper triangle elements to construct the dictionary, and thus the mapping function can be written as:
\begin{equation}
   \varphi (\mathbf{P})= triu(log(\mathbf{P})).
\label{dic-1}
\end{equation}
For patch-based unitary model, $\phi(\cdot)$ is the vectorized mapping function of a unitary dictionary. 

Given the joint statistical and spatial dictionaries, we compute the coefficient vectors $\bm{\alpha}_j$ and $\bm{\gamma}_j$ of the $j$-th query sample with feature representations $\varphi ({\hat{\mathbf{P}}}_j)$ and $\phi({\hat{\mathbf{D}}}_j)$ by solving the following problem:
\begin{equation}\label{CSRGC}
\begin{aligned}
 \underset{\{\bm{\alpha}_j, \bm{\gamma}_j\}}{\operatorname{min}}\;\theta\left\| \varphi \left( {\hat{\mathbf{P}}}_j \right)- [ {\bf{U}_1,\bf{U}_i,...,\bf{U}_n} ]\bm{\alpha}_j \right\|_2^2 \\ 
 +(1-\theta)\left\| \phi \left( {\hat{\mathbf{D}}}_j \right) - [{\bf{V}_1,\bf{V}_i,...,\bf{V}_n}]\bm{\gamma}_j \right\|_2^2 \\
 +{\lambda _1}\left\| \bm{\alpha}_j  \right\|_2^2 + {\lambda _2}\left\| \bm{\gamma}_j  \right\|_2^2 + {\lambda _3}{\left\| {\left[ \bm{\alpha}_j,\bm{\gamma}_j \right]} \right\|_{2,1}},
\end{aligned}
\end{equation}
where $\theta\in(0,1)$ is the weighting parameter defined to balance the scale of statistical model and patch-based model, and $\varphi ( \cdot ),\phi \left(  \cdot  \right)$ are used to map two kinds of representations, respectively.
${\left\| {\left[ \bm{\alpha}_j,\bm{\gamma}_j \right]} \right\|_{2,1}}$ is the $\ell_{2,1} \textnormal{-}$norm for row sparse.

\temp{
By solving the optimization problem in (\ref{CSRGC}), we can get the representation coefficient vectors $\bm{\alpha}_j$ and $\bm{\gamma}_j$ that correspond to the Gaussian and patch-based dictionary models, respectively.
With two coefficient vectors $\bm{\alpha}_j$ and $\bm{\gamma}_j$, we can compute the reconstruction loss of the $j$-th query sample using only one particular sub-dictionaries and corresponded coefficient sub-vectors of the $i$-th class as:
\begin{equation}
\begin{aligned}
{e}_i^{j} = \theta \left\| \varphi \left( {\hat{\mathbf{P}}}_j \right) - \Tilde{\mathbf{P}}_j^{i} \right\|_2^2 + (1- \theta) \left\| \phi \left( {\hat{\mathbf{D}}}_j \right) - \Tilde{\mathbf{D}}_j^{i} \right\|_2^2 \\
+{\lambda _1}\left\| \bm{\alpha}_{j}^{i}  \right\|_2^2 + {\lambda _2}\left\| \bm{\gamma}_{j}^{i}  \right\|_2^2 + {\lambda _3}{\left\| {\left[ \bm{\alpha}_{j}^{i},\bm{\gamma}_{j}^{i} \right]} \right\|_{2,1}}.
\label{cls_m}
\end{aligned}
\end{equation}

Here $\bm{\alpha}_{j}^{i}$ and $\bm{\gamma}_{j}^{i}$ are particular coefficient sub-vectors of the $i$-th class for $J$-th query sample.
$\Tilde{\mathbf{P}}_j^{i} = \bf{U}_i \bm{\alpha_{j}^{i}}$ and $\Tilde{\mathbf{D}}_j^{i} = \bf{V}_i \bm{\gamma_{j}^{i}}$ are reconstructed statistical and spatial representations of the $j$-th query sample, respectively. 
$\bf{U}_i$ and $\bf{V}_i$ are two sub-dictionaries of training samples from the $i$-th class. 
$\bm{\alpha_{j}^{i}}$ and $\bm{\gamma_{j}^{i}}$ are the coefficient sub-vectors containing the sparse code corresponding to the training samples from the $i$-th class.
Moreover, (\ref{CSRGC}) and (\ref{cls_m}) share similar regularized terms to constrain the overall sparsity for sparse coding.
While for classification, we only keep the reconstruction loss term, ignoring the influence of the regularized term on it, thus the reconstruction loss can be rewritten as:
\begin{equation}
\Tilde{e}_{i}^{j} = \theta \left\| \varphi \left( {\hat{\mathbf{P}}}_j \right) - \Tilde{\mathbf{P}}_j^{i} \right\|_2^2 + (1- \theta) \left\| \phi \left( {\hat{\mathbf{D}}}_j \right) - \Tilde{\mathbf{D}}_j^{i} \right\|_2^2.
\label{cls}
\end{equation}



For a visual classification task, the most commonly used algorithm is Nearest Neighbor (NN), aiming to find the closest labeled sample to the current query sample according to the pre-defined metric methods and classify the query sample into the category corresponding to the closest sample.
Inspired by the idea of NN, we assume that features from the same class should be easier to reconstruct, since their feature representations contain similar embeddings, while features from different classes will be more difficult and produce larger reconstruction errors.
For $j$-th query sample, we utilize the reconstruction loss defined in (\ref{cls}) as the measurement for classification and measure similarity in terms of the overall representation of the whole category.
\begin{equation}
   \hat{y}_j = \underset{i}{\operatorname{argmin}}\left\{{\Tilde{e}_{i}^{j}}\right\}, \;\; for \;\; i=1,2,...,N.
\label{classify}
\end{equation}
Here $e_i$ is the reconstruction error of the $i$-th class computed by (\ref{cls}) and $\hat{y}_j$ is the predicted label of $j$-th query sample.

Fig.~\ref{classification} shows how the classification module works. 
Specifically, to classify the query sample, for each class $i$, we only use labeled samples of the corresponding category $i$ for joint sparse representation to reconstruct the query sample.
The query data $({\hat{\mathbf{P}}}_j,{\hat{\mathbf{D}}}_j)$ can then be classified according to the weighted reconstruction error of each class.
}

\section{Algorithm} \label{section4}
We propose a joint sparse representation model for image and image-set classification tasks.
It is obvious that each sub-problem of (\ref{CSRGC}) is convex.
Thus we use alternation minimization to solve the optimization problem.

{{\bf{Update}} $\bm{\alpha}_j,\bm{\gamma}_j$:}
The partial derivatives of the objective function with respect to the $\bm{\alpha}_j,\bm{\gamma}_j$ will be set to 0.

\begin{equation}
\begin{aligned}
\frac{{\partial f}}{\partial \bm{\alpha}_j} &= \theta\left( - 2{{\bf{U}}^T}\varphi \left( {\hat{\mathbf{P}}}_j \right) + 2{{\bf{U}}^T}{\bf{U}}\bm{\alpha}_j \right) \\ &\quad\quad\quad+ 2{\lambda _1}{\bf{I}}\bm{\alpha}_j + 2{\lambda _3}{{\bf{G}}_j}\bm{\alpha}_j = 0.
\end{aligned}
\end{equation}
\begin{equation}
\begin{aligned}
\frac{{\partial f}}{\partial \bm{\gamma}_j} &= (1-\theta)\left( - 2{{\bf{V}}^T}\phi \left( {\hat{\mathbf{D}}}_j \right) + 2{{\bf{V}}^T}{\bf{V}}\bm{\gamma}_j \right) \\ &\quad\quad\quad+ 2{\lambda _2}{\bf{I}}\bm{\gamma}_j + 2{\lambda _3}{{\bf{G}}_j}\bm{\gamma}_j = 0.
\end{aligned}
\end{equation}
where ${\bf{G}}_j$ is a diagonal matrix with the $k$-th diagonal element as $\frac{1}{{2{{\left\| {\left[ {\bm{\alpha}_j}_k,{\bm{\gamma}_j}_k \right]} \right\|}_2}}}$.

Thus we can get the iteration of $\bm{\alpha}_j,\bm{\gamma}_j$ as:
\begin{equation}
    {\bm{\alpha}_j} = {\left( {{{\bf{U}}^T}{\bf{U}} + \frac{{{\lambda _1}}}{{\theta}}{\bf{I}} + \frac{{{\lambda _3}}}{{\theta}}{{\bf{G}}_j}} \right)^{ - 1}}{{\bf{U}}^T}\varphi \left( {\hat{\mathbf{P}}}_j \right).
\label{solutionau}
\end{equation}
\begin{equation}
    {\bm{\gamma}_j} = {\left( {{{\bf{V}}^T}{\bf{V}} + \frac{{{\lambda _2}}}{{1-\theta}}{\bf{I}} + \frac{{{\lambda _3}}}{{1-\theta}}{{\bf{G}}_j}} \right)^{ - 1}}{{\bf{V}}^T}\phi \left( {\hat{\mathbf{D}}}_j \right).
\label{solutionac}
\end{equation}

{{\bf{Update}} ${\mathbf{G}}_j$:}
${\mathbf{G}}_j$ can be updated as follows:
   \begin{equation}
    {{{\mathbf{G}}_j}_{kk}} = \frac{1}{{2{{\left\| {\left[ {\bm{\alpha}_j}_k,{\bm{\gamma}_j}_k \right]} \right\|}_2} + eps}}.
    \label{solutiong}
   \end{equation}
where $eps = 1e\textnormal{-}16$ is an offset to prevent the unsolvable problem of Eq (\ref{solutiong}) in this paper.

\subsection{Complexity Analysis} \label{sec41}
\temp{
We discuss the time and space complexity of our proposed J3S model.
Compared with the steps of the dictionary construction and sparse coding, which require to construct the dictionaries and learn sparse code jointly, the classification part is calculated only based on the results obtained in (\ref{CSRGC}), so we do not consider the impact of the classification part on the complexity of the algorithm here.
As the sub-problems of (\ref{CSRGC}) are all convex and the objective functions are all lower-bounded, the optimization algorithm can converge to a local minimum~\cite{niesen2009adaptive}.
The time complexity consists of the updating of $\bm{\alpha}_j$, $\bm{\gamma}_j$, and ${\mathbf{G}}_j$.
The computational complexity of $\bm{\alpha}_j$ is $O(N^2d_1+N^3)$, and the computational complexity of $\bm{\gamma}_j$ is $O(N^2d_2+N^3)$.
Hence, the main time complexity of the proposed algorithm is $O(s(N^2d_m+N^3))$, where $s$ is the iteration number, $N$ is the number of training samples, $d_1=d(d+1)/2$ and $d_2=p \times p$ are dimensions of two dictionaries, respectively. 
$d$ is the number of channels or samples and $w$ is the patch size, and $d_m$ is the larger feature dimension of two dictionaries.

For space complexity, the proposed J3S model needs to save two dictionaries $\varphi(\mathbf{P})$ and $\phi(\mathbf{D})$ for all samples, a pair of sparse vectors $\alpha_j, \gamma_j$, and the corresponding matrix $\mathbf{G}$ for each query sample. 
The dimensions of statistical dictionary $\varphi(\mathbf{P})$ and unitary dictionary $\phi(\mathbf{D})$ are equal to $d_1$ and $d_2$, respectively.
}

\subsection{A simple and effective strategy for Model Acceleration} \label{sec42}
\temp{
The time and space complexity of the J3S model depends on the dimension of two dictionaries. 
Referring to the process of dictionary construction introduced earlier, we use the lower triangle form to store matrix information of $\varphi(\mathbf{P})$ for dimensionality reduction.
However, the dimension of dictionary is still too high when using deep feature with $d=512$ while the number of training samples $N$ is only hundreds.
In this way, $d_m \gg N$ and the time complexity is close to $O(s(N^2d_m))$, which is not conducive to the application of the proposed algorithm.

To reduce time and space cost, we use the commonly used principal component analysis (PCA) method~\cite{martinez2001pca} to perform dimensionality reduction operations on the two dictionaries to eliminate redundant information between different sizes.
For simplicity, we use $\mathbf{S}_1 \in \mathbb{R}^{N \times d_1}, \mathbf{S}_2 \in \mathbb{R}^{N \times d_2}$ to represent the dictionary of statistical Gaussian model and patch-based model generated from the whole dataset, respectively as
\[
\begin{split}
\begin{aligned}	
{\mathbf{S}_1} &= \left[ {\varphi \left( {{{\mathbf{P}}_1}} \right),\varphi \left( {{{\mathbf{P}}_2}} \right),...,\varphi \left( {{{\mathbf{P}}_N}} \right)} \right]^T,\\
{\mathbf{S}_2} &= \left[ {\phi \left( {{{\mathbf{D}}_1}} \right),\phi \left( {{{\mathbf{D}}_2}} \right),...,\phi \left( {{{\mathbf{D}}_N}} \right)} \right]^T.
\end{aligned}
\end{split}
\]

Specifically, we learn a principal component transformation $\mathit{f_{pca}}: \mathbb{R}^{N \times d_m} \rightarrow \mathbb{R}^{N \times (N-1)}$ to map the data from the original space to a new low-dimension space.
Using the transformation $\mathit{f_{pca}}$, we get new low-dimensional dictionaries $\hat{\mathbf{S}}_1$ and $\hat{\mathbf{S}}_2$ of the statistical and spatial models as:
\begin{equation}
{\hat{\mathbf{S}}_1} = \mathit{f_{pca}}(\mathbf{S}_1), \; {\hat{\mathbf{S}}_2} = \mathit{f_{pca}}(\mathbf{S}_2).
\end{equation}

After PCA operation, we store these two matrices for sparse representation learning. 
For each iteration, we select the columns corresponding to the training samples of the $i$-th class to form the dictionaries $\mathbf{U}_i$ and $\mathbf{V}_i$ to optimize (\ref{CSRGC}).
With PCA, we can reduce the dimensions of two dictionaries $d_1, d_2$ to the same level of $N$, which reduces the time and space cost, \ie the time complexity is reduced from $O(s(N^2d_m))$ to $O(s(N^3))$ with $d_m \gg N$.
The overall optimization procedure is formulated as Algorithm \ref{alg:frame}.


\begin{algorithm}[t]
\caption{Joint statistical and spatial sparse representation.}
\begin{algorithmic}[1]
\REQUIRE ~~\\
{    Training data $\left\{ {{{\bf{X}}_i},{{\bf{X}}_2},...,{{\bf{X}}_n}} \right\},{\bf{X}}_i \in \mathbb{R}^{d \times m_i}$.\\
       A query data ${\hat{\mathbf{X}}}_j$.}
\STATE Construct two dictionaries $\bf{S}_1$ and $\bf{S}_2$ of the whole dataset;
\STATE Adapt PCA to get the low dimension representation $\hat{\bf{S}}_1$ and $\hat{\bf{S}}_2$ of two dictionaries;
\STATE Initialize ${\bf{G}}_j$ as an identity matrix;
\REPEAT
    \STATE update $\bm{\alpha}_j$ according to Eq (\ref{solutionau});
    \STATE update $\bm{\gamma}_j$ according to Eq (\ref{solutionac});
    \STATE update ${\bf{G}}_j$  according to Eq (\ref{solutiong});
\
\UNTIL {convergence criterion satisfied}.
\STATE Classify the query data ${\hat{\mathbf{X}}}_j$  by (\ref{classify}).
\ENSURE ~~\\
The prediction label of the query data ${\hat{\mathbf{X}}}_j$.\\
\end{algorithmic}
\label{alg:frame}
\end{algorithm}
}

\section{Experiments} \label{section5}

We present experimental results on video-based face recognition, material classification, and object recognition tasks to demonstrate the effectiveness of the proposed J3S classification algorithm~\footnote{The reproducible implementations of the J3S algorithms will be made publicly available upon paper acceptance.}.
We conduct experiments on four databases: Flickr Material Database (FMD) \cite{Sharan2009Material}, UIUC Material Database \cite{uiuc}, ETH-80 \cite{eth80} and YouTube Celebrities \cite{ytc}.
FMD and UIUC databases are used for image-based classification, while ETH-80 and YTC databases are used for image set-based classification.
These databases contain samples in different materials, views, illuminations, and even different modalities.
Fig.~\ref{fig:data} shows some sample images with different categories from each database.

\begin{figure}[!t]
\centering
\includegraphics[trim=0.5cm 0cm 0cm 0cm,clip,width=1.0\linewidth]{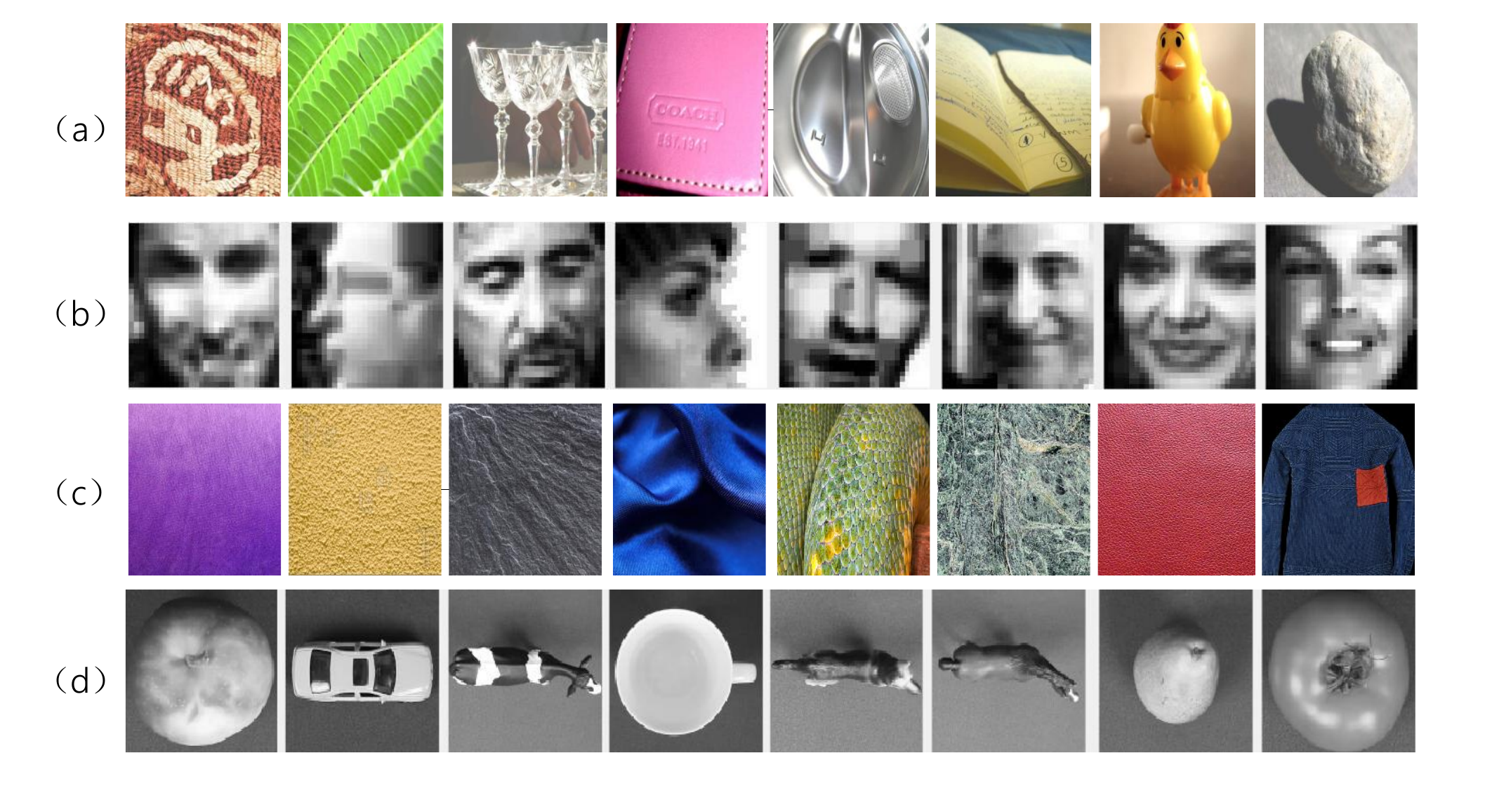}
\caption{Examples of four databases. (a) Flickr Material Database; (b) YouTube Celebrities Database; (c) UIUC Material Database; (d) ETH-80 Database.}
\label{fig:data}
\end{figure}

\emph{Flickr Material} database (FMD) has 10 materials categories of 1000 images in the wild \cite{Sharan2009Material}. 
Each image is selected from \emph{Flickr.com} with variations of illuminations, rotations, and scales. 
We filter images with the \emph{VGG-VD16} model pre-trained on the ImageNet database and employ the output of the last convolution layer as local features with the size of $m_i \times 512$. 
Following \cite{wang2016raid}, we randomly choose 50 images in each category for gallery and the other 50 for probes and repeat this experiment ten times.

\emph{UIUC Material} database contains 216 images of 18 material categories in the wild \cite{uiuc}. 
We obtain the deep features of UIUC material by taking the same measures on FMD. 
We randomly choose half images in each category for the gallery and the other half for probes.

\emph{ETH-80} database contains 80 image sets of 8 object categories \cite{eth80}.
Each category has 10 sub-object with 41 images of different views.
Following \cite{wang2012covariance}, we randomly choose 5 objects as the gallery and the other 5 as probes in each category. 
The size of each image is resized to $20 \times 20$, and the intensity feature is used.
Thus, each image set can be expressed by the matrix of $400 \times 41$.

\emph{YouTube Celebrities} (YTC) database contains 1910 video clips of 47 subjects \cite{ytc} with different numbers of frames in each video.
Following \cite{wang2012covariance,huang2015log}, we use histogram equalization to eliminate light effects in pre-processing step and randomly select 3 videos per subject for the gallery and 6 videos for probes. 
Then, each image is resized to a $20 \times 20$ image with the intensity feature.
Thus each video can be expressed by the matrix of $n_i \times 400$ where $n_i$ is the number of frames in each video.


\begin{table}[t]
\centering
\fontsize{8}{14pt}\selectfont
\begin{tabular}{|c||c|c|c|c|}
\hline
Methods & ETH-80 & FMD & UIUC & YTC \\ 
\hline
AHISD(linear) \cite{cevikalp2010face} & 72.50 & 46.72 & 55.37 & 64.65 \\
\hline
AHISD(non-linear) \cite{cevikalp2010face} & 72.00 & 46.72 & 55.37 & 66.58 \\
\hline
CHISD(linear) \cite{cevikalp2010face} & 79.75 & 47.52 & 65.09 & 67.24 \\
\hline
CHISD(non-linear) \cite{cevikalp2010face} & 72.50 & 63.90 & 65.65 & 68.09 \\
\hline
\hline
MMD \cite{Wang2008Manifold} & 85.75 & 60.60 & 62.78 & 69.60 \\
\hline
MDA \cite{Wang2009Manifold} & 87.75 & 62.50 & 67.13 & 64.72 \\
\hline
SPDML-AIRM \cite{harandi2014manifold} & 90.75 & 63.42 & 74.72 & 67.50 \\
\hline
SPDML-Stein \cite{harandi2014manifold} & 90.75 & 63.80 & 68.24 & 68.10 \\
\hline
LEML \cite{huang2015log} & 93.50 & 66.60 & 69.17 & 69.85 \\
\hline
RMML-SPD \cite{zhu2018towards} & 95.00 & 68.88 & 70.09 & 78.05 \\
\hline
RMML-GM \cite{zhu2018towards} & 93.00 & 69.62 & 76.48 & 69.15 \\
\hline
\hline
CDL-LDA \cite{wang2012covariance} & 94.00 & 76.92 & 78.89 & 70.21 \\
\hline
CDL-PLS \cite{wang2012covariance} & 94.00 & 75.36 & 76.39 & 69.94 \\
\hline
\hline
RSR \cite{harandi2012sparse} & 91.50 & 74.92 & 72.59 & 72.77\\
\hline
KGDL \cite{harandi2013dictionary} & 93.00  & 77.40 & 76.32 & 73.91\\
\hline
\hline
DRM \cite{hayat2014deep} & \bf{\textcolor{red}{98.12}} & N/A & N/A & 72.55\\
\hline
MMDML\cite{Lu_2015_CVPR} & 94.50 & N/A & N/A & 78.5\\
\hline
\hline
J3S w/o Spatial Dict. & 95.25 & \bf{\textcolor{blue}{81.40}} & \bf{\textcolor{blue}{83.43}} & \bf{\textcolor{blue}{82.87}} \\ 
\hline
J3S & \bf{\textcolor{blue}{96.00}} & \bf{\textcolor{red}{82.58}} & \bf{\textcolor{red}{84.07}} & \bf{\textcolor{red}{83.09}} \\
\hline
\end{tabular}
\vspace{0.15in}
\caption{{Classification accuracy (in $\%$) over the four selected databases: AHISD and CHISD are affine subspace based methods; MMD to RMML are nonlinear manifold based methods; CDL is a Gaussian distribution based method; RSR and KGDL are based on sparse coding; DRM and MMDML are deep methods. The best (second best resp.) results are highlighted as \textcolor{red}{Red}  (\textcolor{blue}{Blue} resp.).}}\label{result1}
\end{table}

\subsection{Competing methods}
\temp{
To illustrate the effectiveness of the proposed model, we compare our method with the following representatives of the subspace, non-linear manifold, statistical, sparse representation, and deep based methods.

 \begin{itemize}
   \item         Affine subspace based methods:
                 AHISD and CHISD \cite{cevikalp2010face}.
   \item         Nonlinear manifold based methods:
                 MMD\cite{Wang2008Manifold},
                 MDA \cite{Wang2009Manifold},
                 SPDML \cite{harandi2014manifold},
                 LEML \cite{huang2015log},
                 and RMML \cite{zhu2018towards}.
   \item         Gaussian distribution based methods:
                 CDL \cite{wang2012covariance}.
   \item         Sparse representation based methods:
                 RSR \cite{harandi2012sparse},
                 and KGDL \cite{harandi2013dictionary}.
   \item         Deep based methods:
                 DRM \cite{hayat2014deep},
                 and MMDML \cite{Lu_2015_CVPR}.
 \end{itemize}
}
\subsection{Parameter Setting}
We apply the implementations of competing methods provided by the authors with the default settings suggested by the corresponding papers.
For MMD, the PCA percentage is set to $90\% $.
For MDA, we set the number of local models, between-class NN local models, and the subspace dimension the same as \cite{Wang2009Manifold}.
For SPDML, we implement both SPDML-AIRM and SPDML-Stein versions. In both versions, following \cite{harandi2014manifold}, ${v_w}$ is set as the minimum of the samples in one class. The new dimension of the low-dimensional manifold and ${v_b}$ are tuned by 5-fold cross-validation.
We compare our method with both linear and non-linear versions of AHISD and CHISD \cite{cevikalp2010face}, where $98\%$ energy by PCA is retained in non-linear AHISD and the value of error penalty $C$ in CHISD is set as same as \cite{cevikalp2010face}.
For LEML, $\eta $ is tuned from $1e\textnormal{-}3$ to $1e3$ and the value of $\zeta$ is tuned from $0.1$ to $1$.
For RMML, $\lambda $ is set to $0.1$ and $t$ is tuned from $0.2$ to $0.8$. 
For CDL, the distance metric is learned with linear discriminant analysis (LDA) and partial least squares (PLS) in Hilbert space. The reduced feature dimension is set to $c-1$ for LDA, where $c$ is the number of classes. 
For RSR and KGDL, we use SPAMS as a sparse solver and set other parameters as suggested in the papers. The dimension of the subspace of the Grassmann manifold in KGDL is set to 10.

There are four parameters ${\theta}$, ${\lambda _1}$, ${\lambda _2}$, and ${\lambda _3}$ for our proposed J3S method.
The weighting parameter ${\theta}$ is defined to balance two sparse representation models and adjusted based on different scales of databases.
For some databases, such as the \emph{UIUC} database, which contains only a few labeled samples of each class, the statistical dictionary may be challenging to represent the reliable and complete information of a class. 
Thus, we set a small value $\theta=0.1$ to mitigate the impact of the first term in Eq (\ref{CSRGC}), while we set $\theta=0.6$ for other databases.
All regularization parameters ${\lambda _1}$, ${\lambda _2}$, and ${\lambda _3}$ are all set to $1e\textnormal{-}3$.
Moreover, we implement a common backbone \emph{VGG-VD16} network pre-trained on the ImageNet database for feature extractor in this paper.
The maximum number of iterations is set to 50, and we take an early stop when the difference between loss before and after two iterations is less than $1e\textnormal{-}6$.

\subsection{Image and Image-Set Classification}

Table~\ref{result1} compares the image classification results using the proposed method, as well as all selected competing methods. 
Furthermore, we included two deep learning methods, DRM~\cite{hayat2014deep} and MMDML~\cite{Lu_2015_CVPR}, by quoting the results reported over \emph{YTC} database. 
Note that classic methods randomly choose nine image sets for each class, where three image sets for training and the rest six for testing and report the average accuracy of ten times. 
On the contrary, the selected deep-based models divide the whole database into five folds with nine image sets for each class and keep training the model until convergence while the network input is still based on a single picture.
It is clear that our proposed J3S approach outperforms all competition methods
over the \emph{FMD}, \emph{UIUC}, and \emph{YTC} databases.
For the \emph{ETH-80} database, our method outperforms other competition methods except for DRM, which might due to the way of processing data.
Unlike the J3S model, DRM first computes the LBP features of the training data and generates a subset randomly from the training samples, which enhances the capability of the deep network.
Moreover, during testing, the learned DRM model is used to reconstruct each image of a test image-set sample, and a voting strategy is adopted for classification.
In contrast, our J3S model treats all samples in each image set as a classification object.
Tables~\ref{result1} also show that using joint two dictionaries could help integrate multiple information to facilitate classification tasks.



\subsection{Noisy image classification}
\temp{
We simulate \emph{i.i.d.} Gaussian noise with standard deviation $\sigma$ from $5$ to $20$ for all training and testing data on the \emph{UIUC} and \emph{FMD} databases to generate noisy images for classification.
Table~\ref{result3} and~\ref{result4} show the classification accuracy of two databases under different noise ratios.
The results show that the proposed J3S method outperforms the competition methods subject to noise corruption. 
Also, we can observe a clear downward trend for classification accuracy from the two tables as the noise level increases. 
Simultaneously, our proposed J3S model can still perform better than any other models in all noise levels.
Moreover, we also find that for methods with supervised dimension reduction, \ie SPDML-Stein and CDL-LDA, their performances at a relatively higher noise level $\sigma=20$ are more elevated than performance at a lower noise level $\sigma=15$ on the \emph{UIUC} database.
This is partially due to the fact that models can discard the less critical noise part during the dimensionality reduction process.
}


\begin{table}[t]\footnotesize
\centering
\fontsize{8}{14pt}\selectfont
\begin{tabular}{|c||c|c|c|c|}
\hline
Methods & $\sigma=5$ & $\sigma=10$ & $\sigma=15$ & $\sigma=20$ \\ 
\hline
SPDML-Stein & 66.57 & 67.96 & 64.81 & 65.37 \\
\hline
SPDML-AIRM & 74.26 & 73.06 & 71.02 & 70.37 \\
\hline
LEML & 69.17 & 69.81 & 67.22 & 66.02 \\
\hline
CDL-LDA & 79.63 & 77.96 & 76.85 & 76.94\\
\hline
CDL-PLS & 76.48 & 74.91 & 72.41 & 70.74\\
\hline
J3S & \bf{83.61} & \bf{82.41} & \bf{81.39} & \bf{80.46}\\
\hline
\end{tabular}
\vspace{0.10in}
\caption{Classification accuracy (in $\%$) on noisy data with different noise levels ($\sigma$) of the \emph{UIUC} database.}\label{result3}
\end{table}

\begin{table}[t]\footnotesize
\centering
\fontsize{8}{14pt}\selectfont
\begin{tabular}{|c||c|c|c|c|}
\hline
Methods & $\sigma=5$ & $\sigma=10$ & $\sigma=15$ & $\sigma=20$ \\ 
\hline
SPDML-Stein & 62.86 & 58.54 & 54.94 & 52.12 \\
\hline
SPDML-AIRM & 66.60 & 62.52 & 58.68 & 54.46 \\
\hline
LEML & 66.52 & 63.82 & 59.80 & 56.76 \\
\hline
CDL-LDA & 76.60 & 74.62 & 71.90 & 70.98\\
\hline
CDL-PLS & 74.24 & 71.68 & 69.02 & 66.44\\
\hline
J3S & \bf{82.04} & \bf{80.10} & \bf{76.05} & \bf{74.46}\\
\hline
\end{tabular}
\vspace{0.10in}
\caption{Classification accuracy (in $\%$) on noisy data with different noise levels ($\sigma$) of the \emph{FMD} database.}\label{result4}
\end{table}

\subsection{Ablation Study}
\subsubsection{Weight Analysis}
\temp{
As Eq (\ref{CSRGC}) stated, the weighting parameter $\theta$ is used to balance two dictionary models.
We conduct an experiment to investigate the effectiveness of weighting parameter settings on classification accuracy.
Table~\ref{weight} shows the image classification accuracy averaged over the \emph{ETH-80}, \emph{FMD}, \emph{UIUC}, and \emph{YTC} databases, with different values of $\theta$.
For \emph{ETH-80}, \emph{FMD}, and \emph{YTC} databases, it is obvious that as the weighting parameter of the statistical model increases from $\theta=0.1$, the classification accuracy rate has a significantly increase, which is due to the introduction of higher-order Gaussian information.
Compared with the spatial model, the statistical Gaussian model is more discriminative but still needs the spatial model to capture the local information.
Thus, when the weighting parameter increases to a certain level (\eg $\theta$ from 0.5 to 0.7), continuing to increase will cause the accuracy to fluctuate within a small range.
In contrast, we observe that the classification accuracy on the \emph{UIUC} database becomes worse when increasing the weight parameter $\theta$.
A potential explanation is that, the statistical dictionary may be challenging and unreliable to represent the entire information of a class if only given a few labeled samples such as the $\emph{UIUC}$ database.
Simultaneously, when the weighting parameter $\theta$ increases, the impact of the statistical term on the loss function becomes more significant, so the classification accuracy decreases by about $ 1\% $.

\begin{table}[t]\footnotesize
\centering
\setlength{\tabcolsep}{1mm}{
\fontsize{8}{14pt}\selectfont
\begin{tabular}{|c||c|c|c|c|c|c|}
\hline
Databases & $\theta=0.1$ & $\theta=0.3$ & $\theta=0.5$ & $\theta=0.7$ & $\theta=0.9$ \\ 
\hline
ETH-80 & 94.00 & 95.00 & \bf{96.00} & 96.00 & 95.00  \\
\hline
FMD & 80.58 & 81.86 & \bf{82.50} & 82.36 & 82.50  \\
\hline
UIUC & \bf{84.07} & 83.06 & 83.24 & 83.43 & 83.33 \\
\hline
YTC & 76.70 & 80.92 & 82.70 & \bf{83.09} & 83.01 \\
\hline
\end{tabular}}
\vspace{0.10in}
\caption{Classification accuracy (in $\%$) v.s the weighting parameter $\theta$.}\label{weight}
\end{table}

Moreover, Table~\ref{weight2} shows the classification results of the J3S model w/ and w/o PCA with different values of the weighting parameter $\theta$. 
We can observe that our proposed J3S model achieves the best performance with the same weighting parameter ($\theta=0.1$) under two settings.
Meanwhile, we find that, after PCA dimensionality reduction, the highest accuracy rate has improved slightly from $83.98\%$ to $84.07\%$ while the algorithm complexity has decreased, which demonstrates the effectiveness of the J3S model w/ PCA strategy.
}

\begin{table}[t]\footnotesize
\centering
\setlength{\tabcolsep}{1mm}{
\fontsize{8}{14pt}\selectfont
\begin{tabular}{|c||c|c|c|c|c|c|}
\hline
Settings & $\theta=0.1$ & $\theta=0.3$ & $\theta=0.5$ & $\theta=0.7$ & $\theta=0.9$ \\ 
\hline
J3S w/o PCA  & \bf{83.98} & 83.33 & 82.50 & 82.36 & 82.50  \\
\hline
J3S & \bf{84.07} & 83.06 & 83.24 & 83.43 & 83.33 \\
\hline
\end{tabular}}
\vspace{0.10in}
\caption{J3S model w/ or w/o PCA under different values of the weighting parameter $\theta$ on the \emph{UIUC} database.}\label{weight2}
\end{table}

\begin{figure}[t]
\centering
\includegraphics[width=1.0\linewidth]{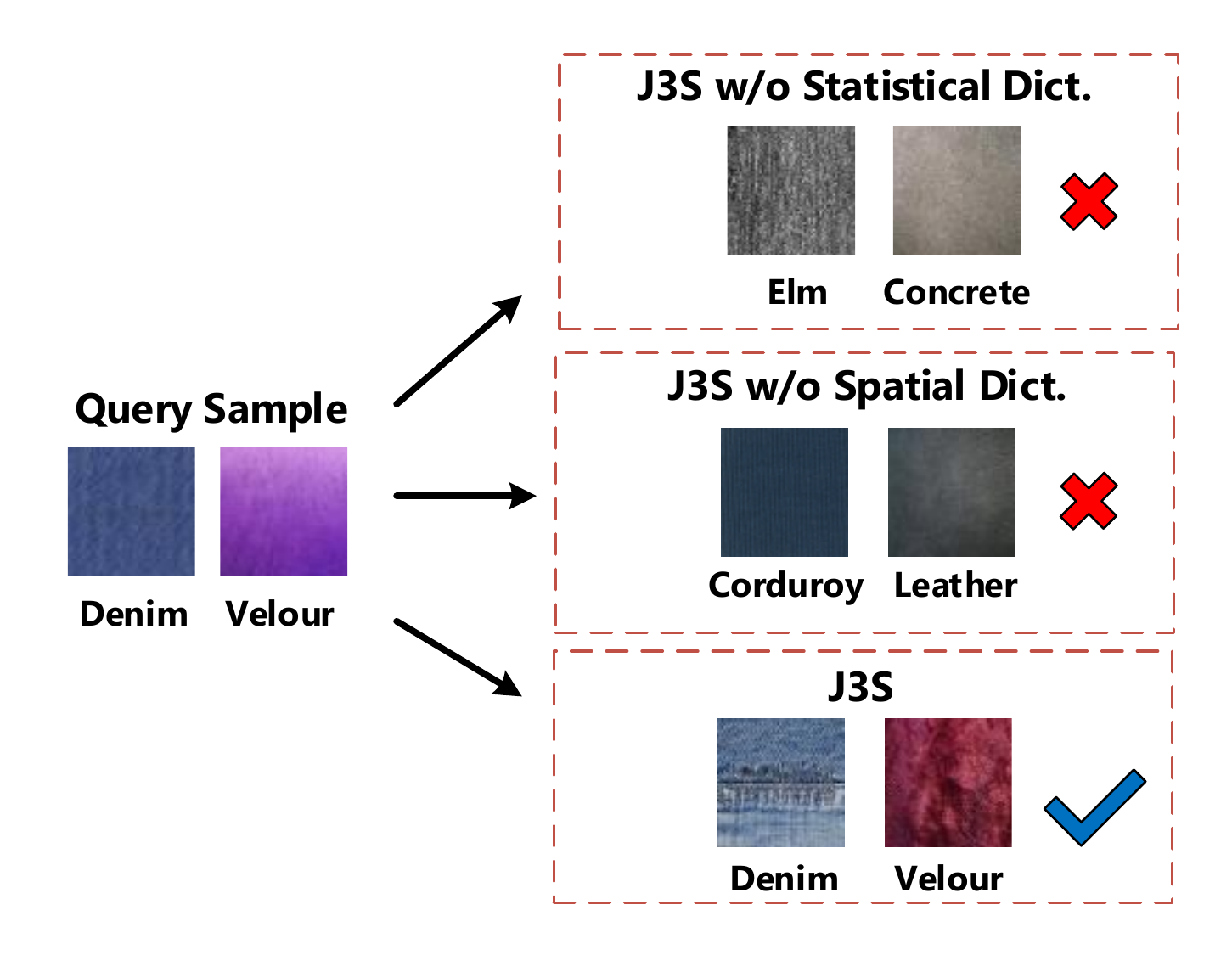}
\caption{Visualization of different results based on the sparse representation of joint dictionaries or two single dictionaries, respectively. Two query samples selected from different categories are shown in the row on the left, and corresponding classification results are shown in the right row.}
\label{visual}
\end{figure}

\subsubsection{Feature Selection for Dictionary Construction}
\temp{
To investigate the proper feature for dictionary construction to extract local information, we try different unitary dictionaries based on deep feature maps, original gray images, and RGB images, respectively.
Table \ref{result2} shows the classification accuracy with different feature selections for our proposed J3S model on clean and noisy image data ($\sigma=20$) of the \emph{UIUC} database.
We can observe that without a unitary dictionary, the model performs worse than other settings under the noisy condition, which can explain the role of the spatial module on the robustness of the model from one side.
Meanwhile, we observe that the accuracy of methods based on a unitary dictionary of gray or RGB image drops smaller than the unitary dictionary based on a deep feature map under the noise condition. 
Since noisy images are fed into a deep CNN structure pre-trained on clean data, it is more difficult to distinguish the noise portion than shallow image features.

\begin{table}[t]\footnotesize
\centering
\fontsize{8}{14pt}\selectfont
\begin{tabular}{|c||c|c|}
\hline
Settings & Acc (Clean) & Acc (Noise) \\ 
\hline
w/o unitary dict. & 83.43 & 80.00\\
\hline
w/ Deep feature based unitary dict. & \bf{84.07} & \bf{80.46}\\
\hline
w/ Gray image based unitary dict. & 83.15 & 80.12\\
\hline
w/ RGB image based unitary dict. & 80.37 & 80.18\\
\hline
\end{tabular}
\vspace{0.10in}
\caption{Classification accuracy (in $\%$) on clean and noisy ($\sigma\!=\!20$) data of the \emph{UIUC} database.}\label{result2}
\end{table}

}

\subsubsection{Convergence Analysis}
\temp{
According to Eq (\ref{CSRGC}), it is easily proved that each part of the objective function is convex.
With alternation minimization, this optimization problem can be divided into two convex problems and solved easily.
Fig.~\ref{convergence} shows that the J3S model can converge within a few iterations with different regularization parameters.
Meanwhile, we can observe that after only one iteration, the J3S model will reduce the value of the loss function to near convergence.
Moreover, compared Fig.~\ref{convergence}$(a)$ with Fig.~\ref{convergence}$(b)$, we find that increasing regularization parameters $\lambda_1$ and $\lambda_2$ for the J3S model will make the convergence more stable but
require more iterations to converge fully, \ie small regularization parameters need more iterations to converge just like $28$ in Fig.~\ref{convergence}$(b)$ than $4$ in Fig.~\ref{convergence}$(a)$. 
A potential explanation is that the regularization parameters $\lambda_1$ and $\lambda_2$ of two sparse models are adopted to control the stringency of sparse vectors $\alpha_j$ and $\gamma_j$ for a query sample with index $j$.
The constraint of the regular terms on the J3S model is proportional to the scale of the corresponding regularization parameters.


}

\begin{figure}[t]
    \centering
	  \subfloat[$\lambda_1=\lambda_2=1e\textnormal{-}3, \lambda_3=0.1$.]{
	  \includegraphics[width=0.47\linewidth]{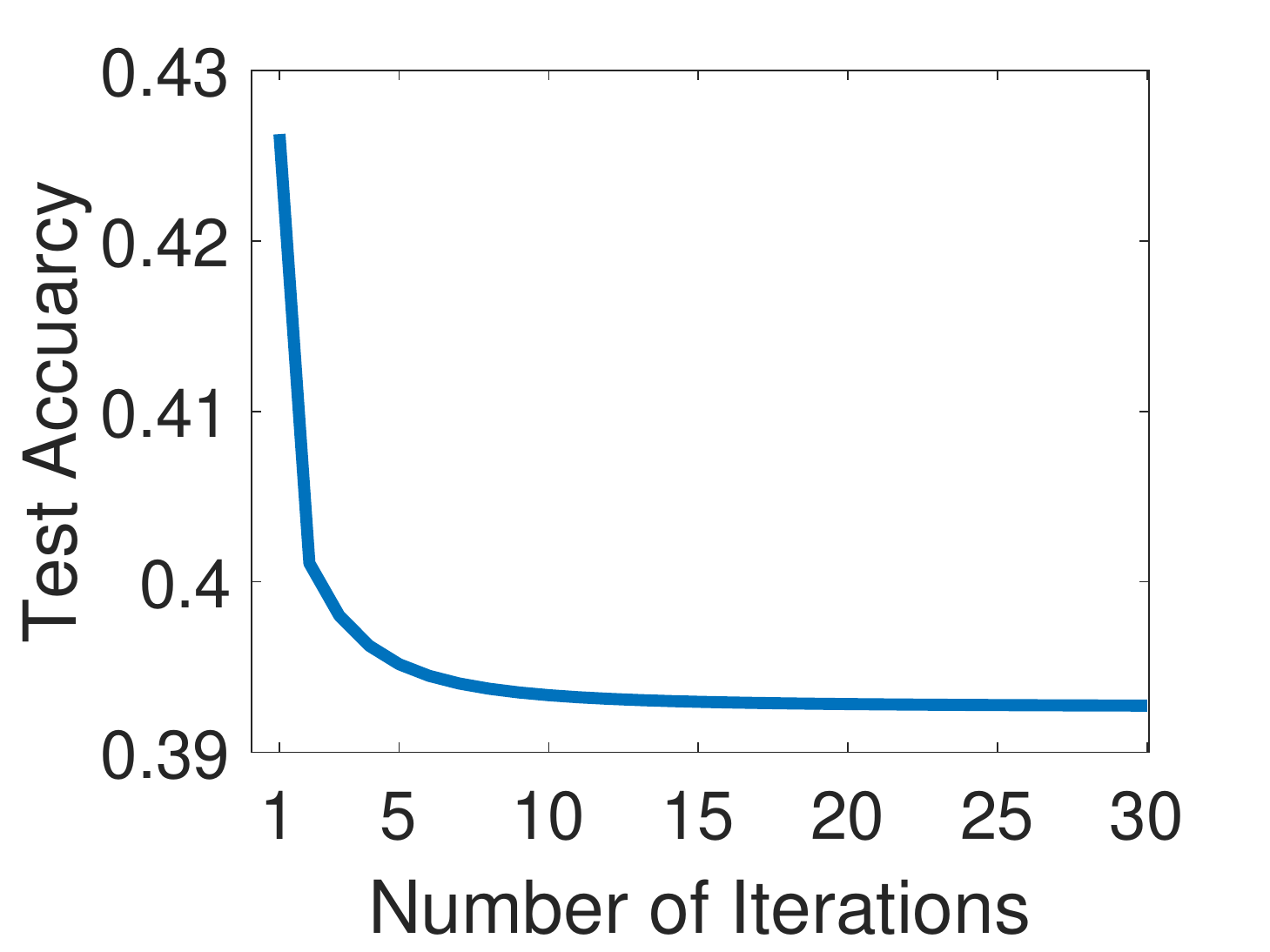}}
    \label{1a}
	  \subfloat[$\lambda_1=\lambda_2=\lambda_3=0.1$.]{
        \includegraphics[width=0.47\linewidth]{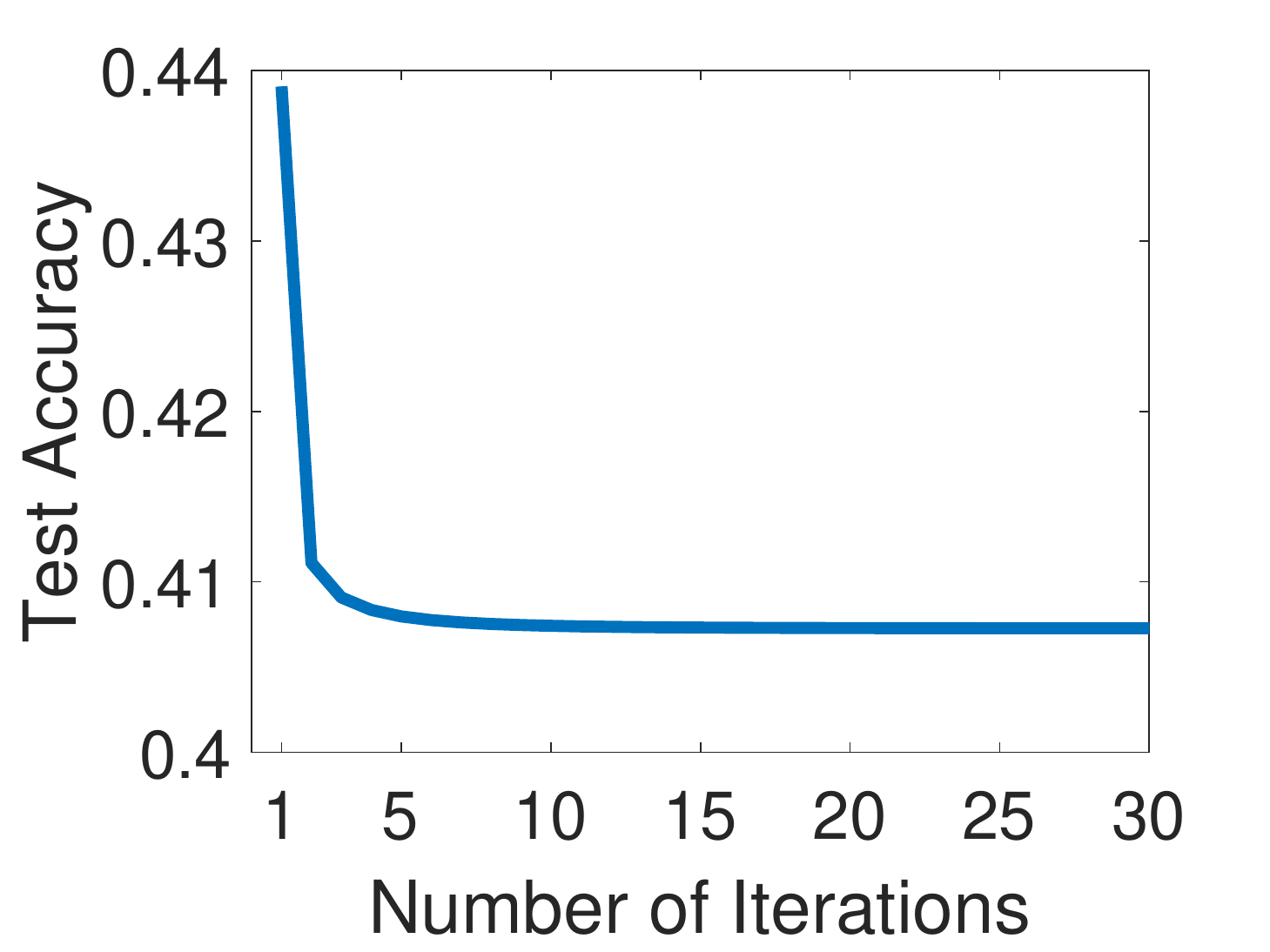}}
    \label{1b}
     \vspace{0.10in}
	  \caption{Convergence curves with different settings of three regularization parameters $\lambda_1$, $\lambda_2$, and $\lambda_3$ on the \emph{UIUC} database. We only show the loss value of 30 iterations on the basis of ensuring the convergence of the model.} 
	  \label{convergence} 
\end{figure}

\begin{table}[t]\footnotesize
\centering
\setlength{\tabcolsep}{1mm}{
\fontsize{8}{14pt}\selectfont
\begin{tabular}{|c||c|c|c|c|c|c|}
\hline
Methods & $K=1$ & $K=2$ & $K=3$ & $K=4$ & $K=5$ & $K=6$ \\ 
\hline
SPDML-Stein & 42.41 & 53.70 & 60.56 & 64.81 & 66.57 & 68.24 \\
\hline
SPDML-AIRM & 48.06 & 62.13 & 68.24 & 70.83 & 72.69 & 74.72 \\
\hline
LEML & N/A & 53.43 & 61.76 & 65.09 & 67.59 & 69.17 \\
\hline
CDL-LDA & 26.48 & 38.89 & 47.22 & 51.56 & 65.19 & 78.89 \\
\hline
CDL-PLS & 55.19 & 67.96 & 72.64 & 74.56 & 75.56 & 76.39 \\
\hline
CNN+Mean+SVM & 59.60 & 70.44 & 75.56 & 77.78 & 78.70 & 81.67 \\
\hline
CNN+Gau+SVM & 61.61 & 72.22 & 77.40 & 79.58 & 81.90 & 84.01 \\
\hline
J3S & \bf{61.94} & \bf{75.56} & \bf{78.89} & \bf{80.83} & \bf{82.41} & \bf{84.07} \\
\hline
\end{tabular}}
\vspace{0.10in}
\caption{Classification accuracy (in $\%$) with only $K$ labeled training samples of each class on the \emph{UIUC} database.}  
\label{result_few}
\end{table}

\subsection{Few-shot Classification}
\temp{
We consider a popular and challenging setting, \ie the few-shot setting for image classification tasks, to investigate the robustness of the proposed J3S model with only a few supervised information.
For few-shot learning, the whole dataset is divided into two non-overlapping label sets, \ie training set and testing set.
Following the meta-learning strategy~\cite{vinyals2016matching}, most existing few-shot methods construct $N$-way $K$-shot tasks on the testing set to evaluate the generalized performance of proposed models trained on the training set.
Here $N$ and $K$ are the number of class and labeled samples, respectively, and $N$ and $K$ are often small values.

Typical few-shot methods follow a methodology that learns the models only from the training set and tests the classification accuracy on the testing set.
Unlike this learning strategy, based on representative learning, the proposed J3S model solves the classification task by utilizing the reconstruction loss computed by two coefficient vectors learned from both labeled training samples and each query sample.
A little different from the typical few-shot setting, we set the value of $N$ to the total number of categories instead of a commonly used fixed value $5$. 
In contrast, we still set $K$ to a small number, identical to the few-shot setting.

Table~\ref{result_few} shows classification accuracy on the \emph{UIUC} database with different numbers of the labeled samples in each class for dictionary learning.  
As mentioned before, the \emph{UIUC} database only has $6$ training samples for each class in total, so the general classification setting with all training samples satisfies one kind of few-shot setting with $K = 6$.
From table~\ref{result_few}, we can observe that our proposed J3S model outperforms the other methods in all few-shot settings, \ie $K$ from 1 to 6.
Note that almost all Gaussian-based models, \ie CDL-LDA, CDL-PLS, two Gaussian-based CNN models, and our proposed J3S model perform well even when using only a few training samples for each class.
The result demonstrates that the global statistical model can implement rich information of a class, enhancing the classification capability in few-shot tasks.
Meanwhile, we observe that our J3S model and the method CNN+Gau+SVM perform better than the other two Gaussian-based methods CDL-LDA and CDL-PLS.
A potential explanation is that, CDL-LDA and CDL-PLS learned only based on the covariance matrix while the J3S model and CNN+Gau+SVM method jointly implement first-order and second-order information, which leads to better accuracy.
Moreover, CDL-LDA performs poorly when $K=1$ because for the 1-shot setting, the feature dimension $\mathbf{D}$ is much larger than the number of training samples $m$ (here $m\!=\!18 \times K\!=\!18$).
After feature projection (so-called dimensionality reduction), the LDA-based model cannot maintain the difference among neighbors and keep the within-class variance to a minimum value for Nearest Neighbor classification. 
In comparison to LDA, PLS has proven to be helpful in this situation as it is not limited by the low discrimination dimensions.

Additionally, as the supervised information decreases, \ie $K$ is selected from $6$ to $2$, the gap between our proposed J3S model and the method CNN+Gau+SVM becomes larger.
This might due to the effectiveness of spatial information when the statistical model cannot provide sufficient information for classification.
However, when $K=1$, the difference among all methods will become minor because only one support sample of each category can be utilized for learning, which results in insufficient information for classification, and the model is vulnerable to bias.
}

\section{Conclusion} \label{section6}

In this paper, we proposed a novel J3S model for robust image and image-set classification. 
Gaussian distribution is used to keep high-order image statistical information, while patch-based sparse representation is used to capture image local structure.
\temp{
A simple and effective dimensionality reduction operation by PCA is utilized to reduce the algorithm complexity.
We conducted experiments over four popular databases 
for clean and noisy image classification tasks.
Moreover, we conducted parameter sensitivity analysis and tested the robustness of the algorithm under the popular few-shot setting.}
Results show that our proposed method achieves superior performance compared to a variety of algorithms under several settings.

\bibliographystyle{IEEEtran}
\bibliography{TMM2020_J3S}
\end{document}